\documentclass[11pt,a4paper]{article}
\usepackage[hyperref]{acl2020}
\usepackage{times}
\usepackage{latexsym}

\usepackage{microtype}
\usepackage{amsmath}
\usepackage{subcaption}
\usepackage[]{algorithm2e}
\usepackage{tabularx}
\usepackage{xcolor}
\usepackage{graphicx}
\usepackage{mathtools}
\usepackage{amsfonts}
\usepackage{amsthm}
\usepackage{textcomp}
\usepackage[linewidth=1pt]{mdframed}

\newtheorem{prop}{Proposition}
\newtheorem{lem}{Lemma}

\aclfinalcopy 
\def\aclpaperid{2354} 

\newcommand{\minihead}[1]{{\vspace{.4em}\noindent\textbf{#1.} }}

\newcommand{\colorred}[1]{{\textcolor{red}{#1}}}

\newcommand{\pref}[0]{p_\mathrm{ref}}
\newcommand{\pmodel}[0]{\hat{p}}
\newcommand{\truncset}[1]{\mathcal{P}_{c,#1}}
\newcommand{\ptrunc}[0]{p_{\mathrm{t}}}
\newcommand{\KL}[0]{\text{KL}}
\newcommand{\EE}[0]{\mathbb{E}}

\newcommand{\SM}[0]{\S}

\title{Improved Natural Language Generation via Loss Truncation}

\author{Daniel Kang, Tatsunori B. Hashimoto}

\date{}

\begin{document}
\maketitle

\begin{abstract}

Neural language models are usually trained to match the distributional
properties of large-scale corpora by minimizing the log loss. While straightforward
to optimize, this approach forces the model to reproduce all variations in the
dataset, including noisy and invalid references (e.g., misannotations and
hallucinated facts). Even a small fraction of noisy data can degrade the performance of log loss.
As an alternative, prior work has shown that minimizing the distinguishability of generated samples
is a principled and robust loss that can handle invalid references.
However, distinguishability has not been used in practice due to challenges in optimization and estimation.
We propose loss truncation: a simple and scalable procedure which adaptively removes high log loss examples 
as a way to optimize for distinguishability.
Empirically, we demonstrate that loss truncation outperforms existing baselines on
distinguishability on a summarization task. Furthermore, we show that samples generated by the loss truncation model have factual accuracy ratings that exceed those of baselines and match human references.

\end{abstract}

\section{Introduction}
\label{sec:intro}

Learning to generate text is a core part of many NLP tasks, including
summarization \cite{nallapati2016abstractive}, image captioning
\cite{lin2014microsoft}, and story generation \cite{roemmele2016writing}. A
common challenge to all these tasks is that references from the training distribution
are not unique and contain substantial variations in phrasing and content
\cite{wiseman2017challenges, dhingra2019handling}. Learning to generate under a
set of diverse and noisy references is challenging as some variations
ought to be learned (e.g., paraphrasing) while others should not (e.g.,
hallucinated facts, ignoring prompts). 

Existing training procedures for models seek to match the underlying
distribution, leading to models that replicate and sometimes even amplify
unwanted behaviors such as hallucination during generation. For example, neural
language models often produce fluent text that is unfaithful to the source
\cite{tian2019sticking, wiseman2017challenges, lee2018hallucinations}. Existing
work \cite{fan2018hierarchical, holtzman2019curious} has primarily addressed these issues by constructing decoders that
implicitly remove unwanted variation when generating (see
\SM\ref{sec:related-work} for a detailed discussion of task-specific losses).

In this work, we argue that this phenomenon is not model specific, but is due to
the widely-used log loss: we demonstrate that log loss is not robust to noisy
and invalid references (\SM\ref{sec:problem-statement}). In particular, log loss
requires that models assign probabilities to \textbf{all} potential test
reference sequences. As a result, log loss is sensitive to outliers: invalid or
noisy references with small probability mass can cause large changes in model
behavior. We show that the brittleness of log loss, together with the noise in
existing generation datasets, lead to low-quality and unfaithful generated text. 

Instead of optimizing log loss, which has little correlation
with model output quality \cite{theis2016note, hashimoto2019unifying,
gamon2005sentence}, recent work on diverse generation models has proposed optimizing for the \emph{distinguishability} of samples from the model and the reference.
Distinguishability provides a natural and appealing
guarantee: samples that are indistinguishable from human generated text will be as
high quality as human generated text. Furthermore, we show that optimizing for distinguishability
is robust in the face of noisy and even invalid data. Despite its appeal, distinguishability
has not been widely used due to statistical and computational challenges. For
example, existing methods that directly optimize for distinguishability have yet to match even naive log loss based baselines \cite{caccia2018language}.

We propose a modification to the log loss, \emph{loss truncation}, that has
the benefits of distinguishability while being efficient to train. Loss truncation is as efficient
to train as log loss, nearly as robust as distinguishability, and provides
distinguishability guarantees via an upper bound. It achieves these properties
by modifying the standard log loss to adaptively remove examples with high log loss. We additionally extend loss
truncation with a \emph{sequence-level} rejection sampling scheme that generates
higher quality sequences by restricting the outputs to be high probability
sequences.


We show that loss truncation with direct and rejection sampling outperforms
standard log loss based generation methods (beam search, full sampling, top-$k$, and
top-$p$ sampling) on distinguishability, as measured by the HUSE score
\cite{hashimoto2019unifying}. We additionally study the factual accuracy of a
summarization system trained on loss truncation and show that our proposed
approach produces summaries which improve upon all baselines (including beam
searched models) and match references on factual accuracy.

\section{Motivation and Problem Statement}
\label{sec:problem-statement}




\minihead{Task and Background}
We consider a natural language generation task with a \emph{conditional
language model}, where we are given a context $x$ drawn from $p(x)$ and our
probabilistic model $\pmodel(y\mid x)$ produces an output $y$ by
approximating a (usually human) reference distribution $\pref(y|x)$.

In order to achieve this, many existing models are trained to minimize the
Kullback-Leibler (KL) divergence,
\begin{align}
\KL(\pref|| \pmodel) = \underbrace{-E_{\pref}[\log \pmodel]}_{\text{log loss}} +  \underbrace{E_{\pref}[\log \pref]}_{\text{negentropy}}. \label{eq:kl}
\end{align}
We refer to the first term of this divergence as the \emph{log loss} of a model.
The second term is commonly ignored as it is a constant with respect to the model.
Minimizing the log loss
has several practical benefits: 1) it is written as an expected loss (and is thus
straightforward to optimize via stochastic gradient descent), 2) it factorizes
across tokens in autoregressive modeling, and 3) it provides a guarantee on a
model's goodness of fit (Eq~\eqref{eq:kl}).

Unfortunately, log loss also suffers from several drawbacks. It is known to have little correlation with a model's sample quality and it can be brittle to invalid references in the training data.

\minihead{Log loss is not robust to noise}
The KL divergence has intuitively correct behavior when each input $x$ has a single correct
reference $y$: it will maximize the probability of the
single correct reference. However, log loss can be problematic when there are
multiple correct references, of which some are invalid or difficult to model.

\begin{figure}
  \includegraphics[width=\columnwidth]{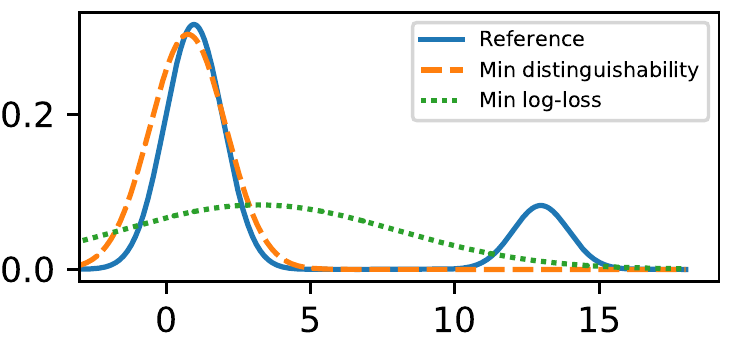}
  \caption{Fitting a mixture of Gaussians with a single Gaussian using
  distinguishability (TV) and log loss (KL). As shown, log loss is extremely sensitive to
  outliers, resulting in poor estimation.}
  \label{fig:log-loss-robust}
\end{figure}

In particular, log loss is sensitive to invalid or noisy data because it requires
that the model assign high probabilities to \emph{all} potential references. Log loss is
unbounded above: a model assigning zero probability to even a single reference
makes the model incur an infinite overall loss.

We show a well-known example of this behavior with synthetic data. We
consider fitting a single Gaussian to a
mixture of two Gaussian in Figure~\ref{fig:log-loss-robust}. The reference distribution (blue)
has a valid set of references at zero as well as variation that the model does
not expect (e.g., invalid or noisy references) on the right. Minimizing the log
loss results in a suboptimal model that is forced to span both groups.
Furthermore, post-hoc processing the model does not help, as even the most
likely output under the log loss trained model (\texttildelow3) has low
probability under the reference distribution.

In natural language generation, training sets can contain invalid or poor
quality references. As such, these types of problems manifest themselves in tasks such as summarization (hallucinating facts), story generation (ignoring prompts and constraints), and captioning (ignoring parts of the image).

Much of the existing literature on faithful generation has focused on designing better models for \emph{valid} references (via copying or attention constraints), but the example in Figure~\ref{fig:log-loss-robust} shows that this alone may not be sufficient. The Gaussian `model' in this case perfectly fits the mixture component at zero but is still brittle because it cannot simultaneously fit the other group of (invalid) samples.
Resolving this will require either a model which is designed explicitly to capture \emph{invalid} references or a loss function that can ignore them.

\paragraph{Case Study: Hallucination in Summarization}


We show that low-probability reference sequences (e.g.,
Figure~\ref{fig:log-loss-robust}) are pervasive by examining the Gigaword
summarization dataset \cite{rush2017neural}. We manually classified 300 titles
into two categories: 1) requires hallucinating new facts and 2) directly
entailed from the context. We show an example of a reference that requires hallucination
in Figure~\ref{fig:new-fact-ex}. In this example, a model that assigns
high probability to the new fact (Thursday) must also frequently hallucinate
dates on other examples. 

\begin{figure}[t!]
  \begin{mdframed}
  \textbf{Context:} For the first time in five years, Microsoft corp.~is finally
  unveiling a new system for operating personal computers. \\
  \textbf{Title:} Microsoft Makes \colorred{Long-Awaited} Software Upgrade Available to
  \colorred{Businesses Thursday}.
  \end{mdframed}
  \vspace{-0.5em}
  \caption{Example of an article title from the Gigaword dataset
  that requires hallucinating new facts such as `Thursday' (colored red).}
  \label{fig:new-fact-ex}
  \vspace{-0.5em}
\end{figure}

We show the fraction of examples in each category in
Table~\ref{table:hallucination-fraction}. As shown, \emph{35\%} of titles
require hallucinating new facts. Others have found this phenomenon to be
pervasive in other datasets \cite{kryscinski2019neural}, including the CNN/DM
dataset \cite{see2017get}.

Studying the log loss of these examples\footnote{The log loss was computed from
a standard language model, see \SM\ref{sec:eval} for details.}, we note that the average log loss of
titles that require new facts is over 1.7$\times$ the average loss of the titles
that are directly entailed (Table~\ref{table:hallucination-fraction}) and the
high-loss examples are clearly dominated by examples which require hallucination
(Figure~\ref{fig:loss-hist}). In fact, we find that over 80\% of examples with
greater than 40 log loss requires some form of hallucination. 

These statistics are similar to the toy example we presented earlier in Figure~\ref{fig:log-loss-robust}. A small but nontrivial fraction of invalid and unexpected data
force the model to incur high losses. Much like in the earlier example, we can see that
a model which aims to have low log loss on this dataset must spend a substantial amount of effort learning to hallucinate.

\begin{table}[t!]
  \centering
  \small
  \begin{tabular}{l|c|c}
    & New facts & Directly entailed \\
    \hline
    Percent & 35\%   & 65\% \\
    Avg. log loss & 34.3 & 20.5
  \end{tabular}
  \vspace{-0.5em}
  \caption{Fraction of the data and log loss of titles that require
  hallucinating new facts (left column) and titles that are entailed from the
  context (right column). As shown, \emph{35\%} of titles require
  hallucinating new facts and the average log loss of titles requiring new facts
  is over 1.7$\times$ the loss of the directly entailed sequences.}
  \label{table:hallucination-fraction}
\end{table}

\begin{figure}
  \includegraphics[width=\columnwidth]{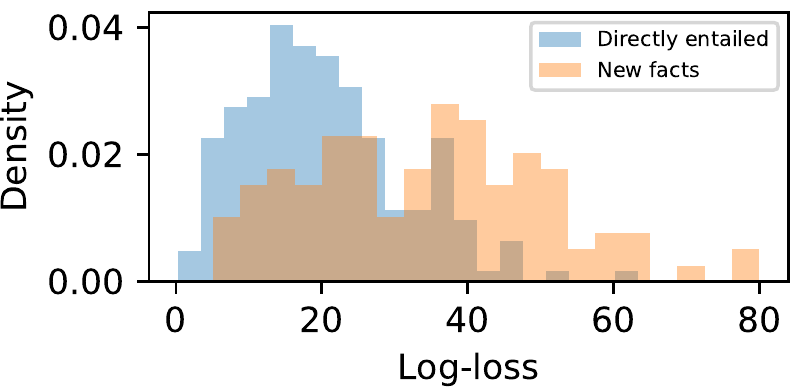}
  \vspace{-0.5em}
  \caption{Normalized histogram of log losses for titles that require
  hallucinating new facts compared to those that can be directly entailed. As
  shown, titles requiring new facts incur significantly higher loss and more
  than 80\% of examples with greater than 40 log loss require hallucinating new
  facts.}
  \label{fig:loss-hist}
\end{figure}


\minihead{Distinguishability}
Given that large-scale data will inevitably contain annotation errors and noise,
we might ask whether there are effective alternatives to the KL divergence for
training models. The distinguishability of samples from a model compared to the
reference is one such objective. Distinguishability has recently gained attention as a way to learn and evaluate models based on both sample quality and diversity \cite{hashimoto2019unifying, zhou2019hype,
zellers2019defending, gehrmann2019gltr}. We show that this objective also
serves as a naturally robust alternative to the KL divergence for learning
language models. Unfortunately, directly optimizing for distinguishability
(e.g., via generative adversarial networks) is challenging \cite{caccia2018language} and we show this
works poorly in practice (\SM\ref{sec:eval}).

Distinguishability is defined as the error rate of an optimal classifier which
seeks to distinguish samples from both the model and reference, and we will
formally define this via the mixture
\[
  y | x, z \sim
    \begin{cases}
      \pref(y|x) & \textrm{if } z = 1 \\
      \pmodel(y|x) & \textrm{if } z = 0
    \end{cases}
\]
where $z \sim \textrm{Bernoulli}\left(\frac{1}{2}\right)$. We can now define $L^*$ to be twice the optimal error in identifying samples from the model
\begin{equation}
  L^* :=
    2 \inf_{f\in{\mathcal{X}\times \mathcal{Y}\to[0,1]}} \mathbb{P}[f(x, y) \neq z]\label{ref:eqloss}
\end{equation}
Our measure of distinguishability, the \emph{total variation (TV) distance}, is a linear function of this error
\[
|\pmodel - \pref|_{TV} = 1 - L^*
\]
where $\pmodel$ and $\pref$ refer to the joint distributions $\pmodel(y|x)p(x)$
and $\pref(y|x)p(x)$ for brevity.
Note that distinguishability is inherently \emph{robust} to the addition of
\emph{any} small fraction of noisy data \cite{donoho1988automatic}. Unlike the log loss,
the model's loss on an example for TV is upper bounded by $1$
(Eq~\ref{ref:eqloss}). We show an example of TV's robustness in
Figure~\ref{fig:log-loss-robust}, where a small amount of noise does not
substantially affect the learned distribution.

\minihead{Log loss as a surrogate for distinguishability}
Distinguishability is both robust and provides sample quality guarantees, but is
challenging to optimize \cite{caccia2018language}. One approach to optimize for
distinguishability is to find an appropriate \emph{surrogate loss} which serves
as an upper bound. This is analogous to the use of logistic or hinge losses as a
way to optimize for classification accuracy. For log loss, \emph{Pinsker's
inequality} \cite{csiszar2011information} relates the KL divergence and
distinguishability as
\begin{equation}
|\pmodel - \pref|_{TV}^2 \leq \frac{1}{2} \cdot \KL(\pref || \pmodel). \label{eq:pinsker}
\end{equation}
This explains the empirical success of log loss in low-uncertainty situations, where $\KL$ is sufficiently small and this bound becomes tight.

Our approach will be to modify the log loss slightly by truncating the
distribution. This truncated loss will be as easy to optimize as log loss, while
being more robust and providing a tighter variant of Pinsker's inequality.






\section{Loss Truncation}
\label{sec:technique}



\minihead{Intuition}
We would like the model to ignore data that would force it to
unnecessarily hallucinate at test time. Concretely, recall the toy example 
(Figure~\ref{fig:log-loss-robust}); there is a set of invalid references that force the model
to be degenerate. If we could remove these these invalid references by
truncating the distribution, the resulting model would be high quality.
We can show that this intuition
is theoretically justified, and that truncating (i.e., removing) an appropriate $c$-fraction of the data provides tighter bounds on the distinguishability of the model.

\minihead{Improved log losses for distinguishability}
We will demonstrate that log loss with an appropriate $c$-fraction of the data removed provides guarantees on distinguishability.
We will define the set of \emph{truncated} distributions as the set of distributions with any $c$-fraction of data removed
\[\truncset{p} := \left\{q_0: p = (1-c)q_0 + c q_1 \text{ for some }q_1 \right\}.\]
A simple lemma shows that that all elements in $\truncset{p}$ are $c$-close to $p$
in TV (Appendix~\ref{sec:proofs}).

Now we state our main result,
\begin{prop}
    For any $c \in [0,1]$ and $\ptrunc \in \truncset{\pref}$,
    \[|\pmodel - \pref|_{\text{TV}}^2 \leq \frac{1}{2}\KL(\ptrunc || \pmodel) + 2c + c^2\]
  \end{prop}
See Appendix~\ref{sec:proofs} for the proof.
Namely, distinguishability is bounded by the log loss with respect to the
truncated distribution and a small constant. Furthermore, this upper bound is valid for
\emph{any} $c$, although different $c$ will change the tightness of the bound and
produce different models.

This truncated bound can be substantially tighter than Pinsker's inequality.
Consider for example a model that can perfectly capture $(1-c)$ fraction of
the data, but $c$-fraction of the reference outputs cannot be generated by the
model and receive probability zero. In this case, the distinguishability
(TV) is $c$, the KL divergence is \emph{infinite}, while our
truncated bound is $\sqrt{c^2+2c}$. This suggests that appropriately
truncating high-loss examples makes log loss robust and allows us to use log loss as a surrogate for distinguishability, even in the presence of invalid and noisy references.


\minihead{Loss truncation}
Given that the log loss on any $c$-fraction of the data is a surrogate loss for distinguishability (Eq~\eqref{eq:bound}), a key parameter to optimize is the
truncated distribution $\ptrunc$. An oracle
solution would exhaustively search over $\ptrunc$ and which data to drop. However, exhaustively
searching through $\truncset{\pref}$ is a combinatorial optimization problem and
infeasible. Our approach will be to optimize $\ptrunc$ with a heuristic. The
truncated objective takes the form of a log loss and negative entropy term,
\[-\EE_{\ptrunc}[\log\pmodel(y\mid x)] + \EE_{\ptrunc}[\log\ptrunc(y\mid x)]\]
and we will select $\ptrunc$ by dropping the examples with the highest log loss,
treating the negative entropy term as being upper bounded by zero.

This heuristic is straightforward to compute, provides an upper bound on distinguishability, and matches our earlier observation that high-loss examples are correlated with invalid examples we would like the model to ignore (see Table~\ref{table:hallucination-fraction}).

\begin{figure}
  \includegraphics[width=\columnwidth]{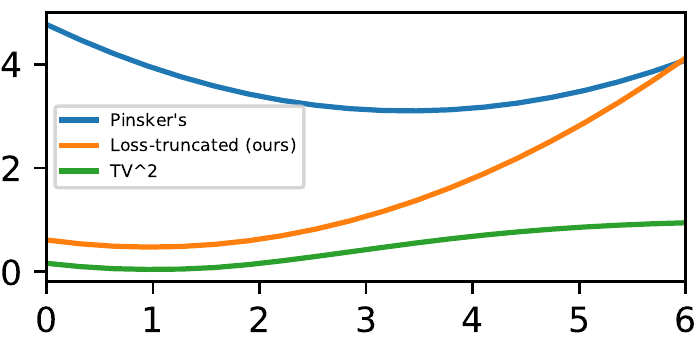}
  \caption{Pinsker's inequality, our bound, and the total variation
  squared of parameter estimates for different parameter estimates ($c=0.2$). As
  shown, loss truncation can significantly improve bounds over Pinsker's
  inequality and, in this case, has a nearly identical minimizer to directly
  minimizing total variation.}
  \label{fig:bound-tightness}
\end{figure}

As an example of how our heuristic can improve estimation and tightness in bounds,
consider the earlier toy example in Figure~\ref{fig:log-loss-robust}. In this
example, we find the optimal mean for a single Gaussian with fixed variance
which fits mixture of two Gaussians. Figure~\ref{fig:bound-tightness} shows the
objective function value implied by the TV loss, log loss (Pinsker's bound), and our $c$-truncated bound
as a function of the Gaussian mean.
We find that log loss
provides an upper bound on distinguishability (via Pinsker's inequality) but is
loose and results in a low quality estimate. In contrast, $c$-truncation
results in a nearly identical minimizer as directly minimizing TV.

\section{Implementing Truncation}
\subsection{Training}

Our algorithm has three components at training time. First, it trains a model on
all the data using standard hyperparameters, which we refer to as
``hotstarting'' the model. Second, it tracks a running estimate of the
$1-c$ quantile of the losses during training. Third, it performs gradient updates on examples that
are below the current $1-c$ quantile estimate. We present
the pseudocode in Algorithm~\ref{alg:loss-trunc} and describe each step in
detail below.\footnote{Our code is available at
\url{https://github.com/ddkang/loss_dropper}.}

\minihead{Hotstarting}
First, our algorithm hotstarts the model (\texttt{hotstart}($M$) in
Alg.~\ref{alg:loss-trunc}) by training with the standard log loss.
Hotstarting address two challenges in optimizing the truncated loss. First,
losses are uninformative at the start of training so truncating examples based
on these losses will result in dropping valid examples. We have
empirically found that truncating after hotstarting primarily drops invalid
references, which avoids this problem.
Second, hotstarting allows the model to transfer information from the entire
dataset to the clean $1-c$ fraction of the data. Examples that cause a model to
hallucinate may still contain valid information about the fluency of a sentence, which hotstarting
can capture. This is effectively pretraining our model on the
entire data before learning to generate on the clean subset. We have found this
procedure to be effective in practice.

\minihead{Quantile estimation}
Second, our algorithm keeps track of the $1-c$ quantile over
the distribution of losses. For each new minibatch $B$, we update an online
estimate of the $1-c$ quantile (\texttt{estimateQuantile}($M,B$) in
Alg.~\ref{alg:loss-trunc}). To estimate this quantile, our algorithm
constructs a histogram over the last 10,000 examples seen during training and
estimates the empirical $1-c$ quantile every 10,000 examples.\footnote{For
datasets with fewer than 10,000 examples, we can perform this procedure over the
entire dataset.} 

\minihead{Loss dropping}
Third, our algorithm will perform minibatch stochastic gradient descent while excluding
examples that have losses above the current top $1-c$ quantile estimate $q$
(\texttt{truncatedUpdate}($M, B, q$) in Alg.~\ref{alg:loss-trunc}). Dropping can be accomplished in automatic differentiation packages
(e.g., Tensorflow and PyTorch) by setting the loss on the given example to zero.

\subsection{Generating High-Probability Samples}
Thus far, our goal has been to robustly learn the
underlying distribution. However, in some cases, a user may wish to only
generate high confidence sequences, which will ideally correspond to high
quality sequences.

To generate such samples, we propose \emph{sequence-level rejection sampling}.

Recall that our truncation heuristic selects for the $1-c$ quantile of the distribution. For a user-defined level $\alpha$, our rejection sampling scheme will aim to
generate samples from the $1-c \cdot \alpha$ quantile.

To perform rejection sampling, given a model and a user-defined rejection level
$\alpha$, we first sample $N$ sequences (e.g., titles in a
summarization task). Then, we sample a random sequence from the $\alpha \cdot N$
smallest samples as measured by log loss. Ideally, this procedure will return a
sample in the $1-c \cdot \alpha$ quantile of $\pref$.


We show that rejection sampling can outperform baselines in generating factual
summaries (\SM\ref{sec:eval}). We further show examples of selected and
rejected samples in Appendix~\ref{sec:generation-ex}.

\begin{algorithm}[t]
  \KwData{Model $M$, $c$ fraction to drop, $T$ iterations}
  $M \gets $ \texttt{hotstart}($M$) \;
  \For{$i \gets 0$ \KwTo $T$ }{
    $B \gets \mathrm{\texttt{minibatch()}}$ \;
    $\mathrm{q=\texttt{estimateQuantile}}(M, B) $ \;
    $\mathrm{M=\texttt{truncatedUpdate}}(M, B, q)$\;
  }
  \caption{The proposed loss truncation procedure with three components (see main text for details for each component).}
  \label{alg:loss-trunc}
\end{algorithm}

\begin{table*}[t!]
  \centering
  \small
  \begin{tabular}{c|c|c|c|c|c|c|c}
    & Loss trunc. & Trunc+reject ($\alpha=0.1$) & Full samp. & Beam & top-$k$ ($k=100$) & top-$p$ ($p=0.9$)  & GAN \\
    \hline
    HUSE   & \textbf{0.58} & 0.04 & 0.55 & 0.04 & 0.32 & 0.32 & 0.003 \\
    HUSE-D & 0.88 & 0.12 & \textbf{0.98} & 0.18 & 0.59 & 0.65 & 0.25 \\
    HUSE-Q & 0.70 & \textbf{0.92} & 0.58 & 0.86 & 0.73 & 0.67 & 0.75
  \end{tabular}
  \vspace{-0.5em}
  \caption{HUSE, HUSE-D, and HUSE-Q scores for loss truncation and baselines. As
  shown, loss truncation outperforms all baselines on HUSE score.}
  \label{table:huse-fu}
\end{table*}

\section{Evaluation}
\label{sec:eval}

\subsection{Experimental Setup}

\minihead{Dataset and Task}
We primarily evaluate loss truncation on abstractive summarization in the form of
generating news headlines from an article. We selected this task to highlight that loss truncation can improve sample quality and factual accuracy, while also achieving the secondary goal of diversity for abstractive systems \cite{see2017get,kryscinski2019neural}.

We evaluated on the Gigaword summarization task \cite{rush2017neural} as in
\citet{gehrmann2018bottom}. While there are other summarization datasets, we
chose Gigaword for the following reasons. First, it is large enough that sample quality defects are not caused by a lack of data.  Second, the dataset is
structured so that neither model nor computation is the bottleneck in
performance: the standard sequence-to-sequence models are competitive on the
Gigaword dataset. Third, while Gigaword dataset is known to have noise, this matches the behavior of existing annotation errors \cite{beigman2009learning, klebanov2010some} and uncertainty \cite{kryscinski2019neural}.

To show that loss truncation is applicable beyond summarization, we also
performed a preliminary evaluation of our approach on the E2E NLG task. In E2E,
the goal is to generate restaurant reviews from meaning
representations~\cite{dusek2019e2e}.

\minihead{Model and Baselines}
We used a standard LSTM architecture with global attention for summarization
that has been used for the Gigaword summarization task in the past
\cite{gehrmann2018bottom}. The learning rate and hyperparameters are given in
Appendix~\ref{sec:hyperparameters}. For the E2E task, we use a standard model
with the exact settings as in \citet{puzikov2018e2e}.

For loss truncation on Gigaword, we used $c=0.6$. We matched the total number of training
steps when training via loss truncation (including the hotstart) and standard
log loss. We sampled from the full model distribution for loss truncated models
except when rejection sampling.

As baselines on Gigaword, we generate from the log loss trained language model using several
decoders that have been reported to mitigate low-quality outputs such as beam
search, top-$k$ sampling \cite{fan2018hierarchical}, and top-$p$ sampling
\cite{holtzman2019curious}. We also evaluate directly sampling from the
probabilistic model in order to estimate overall distinguishability and
understand the diversity-quality trade-offs of each model. 

Finally, on Gigaword, we also compared against a recent generative adversarial
network (GAN) model with a publicly available
implementation \cite{wang2018learning}.

\minihead{Human-evaluation metrics}
We evaluate whether loss truncation improves model distinguishability on
summarization by
measuring the HUSE estimator for TV \cite{hashimoto2019unifying}.
HUSE measures distinguishability by learning a classifier over the
log-probabilities and human evaluation scores over both samples from the model
and references. We also use HUSE to evaluate the quality-diversity tradeoffs of
the models by estimating both HUSE-Q (which measures quality via human judgement) and HUSE-D (which measures diversity via statistical evaluation).

In order to assess whether this leads to improvements in the faithfulness of samples, we  measure whether loss truncation reduces the number of \emph{factually inaccurate} outputs from the model via a crowdsourced survey. We designed our prompt based on earlier factual accuracy human evaluation \cite{novikova2017we} and measured whether the original article contained all of the information given in the generated title.

We describe the crowd worker setup in Appendix~\ref{sec:worker-prompts}.

\minihead{Automated metrics}
While human evaluation is our primary metric of evaluation as it is considered
gold-standard, we additionally evaluate on automated metrics to contextualize
our human evaluation results. We measure ROUGE-L \cite{lin2003automatic} for
summarization and BLEU score \cite{papineni2002bleu} for E2E.

\subsection{Loss Truncation Outperforms Baselines on HUSE}

Using the HUSE score to measure the TV distance, we assessed whether loss truncation successfully improved our model
in terms of distinguishability compared to log loss. As shown in
Table~\ref{table:huse-fu}, loss truncation outperforms all baselines on HUSE
score (including the original log loss model \texttt{Full samp}), suggesting the
truncated model is a better language model than the log loss model as measured
by distinguishability.

We find that that loss truncation improves over the log loss by increasing the
generation quality (HUSE-Q) by 12\% without substantially lowering diversity
(e.g., memorizing examples from the training set). These results affirmatively
answers an open question posed by \citet{hashimoto2019unifying} on whether it is
possible to obtain models that improve the quality while maintaining overall
distinguishability compared to log loss trained models. Post-hoc modification of
the log loss model's distribution by removing unlikely words using either
top-$k$ or top-$p$ sampling result in substantial losses in HUSE due to losses
in diversity.

We further considered matching the entropy of the loss truncation model with
top-$k=100$ and top-$p=0.9$ (Appendix~\ref{sec:hyperparameters}). At a fixed
entropy, loss truncation can outperform on HUSE by up to 26\%.

Comparing models with high sample quality, loss truncation with rejection
sampling improves upon all baselines (including beam search) in terms of raw
human quality evaluation (HUSE-Q), and we see that the Pareto frontier of
truncation and rejection sampling (which can be achieved via ensembling)
dominates the baselines on \emph{both} quality and diversity
(Figure~\ref{fig:huse-2d}).
Rejection sampling decreases overall HUSE score because it is designed to only
return high quality samples (i.e., high HUSE-Q): this comes at the cost of
reduced diversity, so overall HUSE score suffers.

The results amongst our baselines recapitulate known results for the
quality-diversity tradeoffs of existing methods. Beam search has high sample
quality, but low diversity; top-$k$ and top-$p$ samplers provide diversity gains
over beam search; and GANs generally underperform
well-tuned log loss based models on both diversity and quality.

\subsection{Loss Truncation with Rejection Sampling Produces High Quality
  Outputs}

We now ask whether improvements in distinguishability (as measured by HUSE) for the loss truncation
model translate to practical improvements in sample quality, such as the factual accuracy
of generated outputs in summarization. We evaluate this through a crowdsourced study on factual accuracy.


Since we are interested in studying whether our model can produce high quality
samples, we used rejection sampling with  $\alpha = 0.1$ to obtain high-quality samples
from the model. We compare this to the log loss model with baseline decoders. For
the top-$p$ and top-$k$ sampling decoders that have quality-diversity tradeoffs, we
select $k$ and $p$ such that the entropy of the sampling distribution matches our rejection
sampling approach (see Appendix~\ref{sec:hyperparameters} for details). 

\begin{figure}[t]
  \includegraphics[width=\columnwidth]{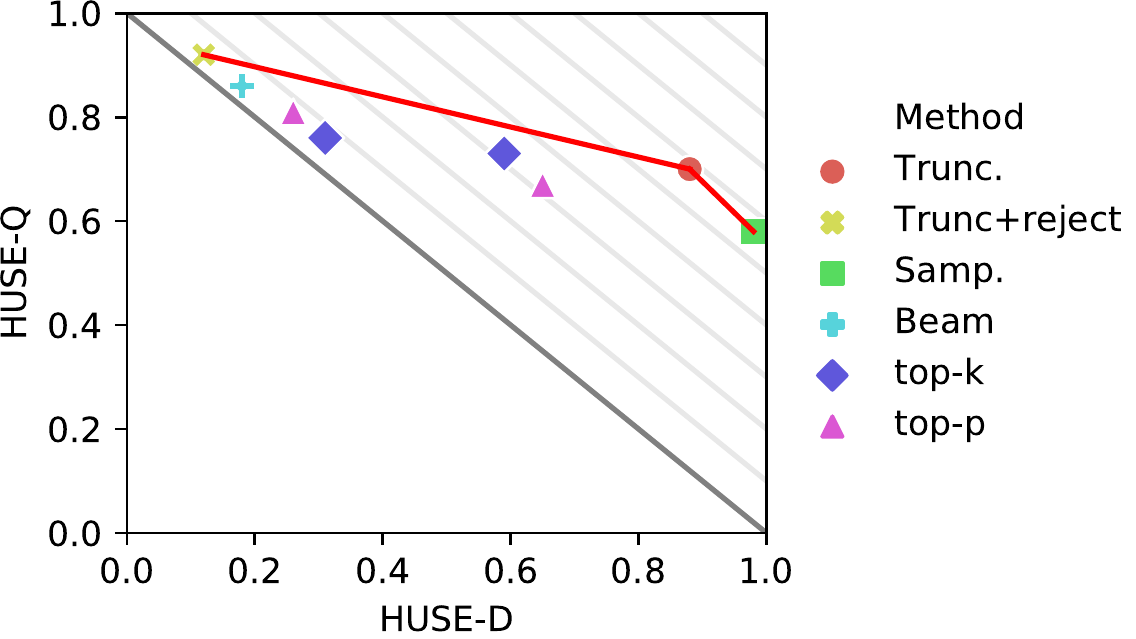}
  \caption{HUSE-D vs HUSE-Q for loss truncation, truncation + rejection sampling, and
  baselines. The red line shows the best achievable frontier via ensembling. Truncation and rejection outperform all baselines.}
  \label{fig:huse-2d}
\end{figure}


\begin{table}[t!]
  \centering
  \small
  \begin{tabular}{l|l}
    Condition & Mean score \\ \hline
    Human & \textbf{3.63} $\pm$ 0.05 \\
    \hline
    Truncation + Rejection ($\alpha=0.1$) & \textbf{3.79} $\pm$ 0.06 \\
    Beam & 3.51 $\pm$ 0.05 \\
    top-$p$ ($p=0.4$) & 3.42 $\pm$ 0.05 \\
    top-$k$ ($k=2$)   & 3.29 $\pm$ 0.05 \\
    Sampling          & 2.96 $\pm$ 0.05
  \end{tabular}
  \vspace{-0.5em}
  \caption{Mean scores and standard errors of factuality in generated news
  titles given articles.  As shown, rejection sampling outperforms all
  baselines and matches the human reference score.}
  \label{table:summ-fact}
\end{table}

To measure factual accuracy, we asked crowd workers how much information in the
generated titles was contained in the article in a similar fashion to
\citet{novikova2017we}. Table~\ref{table:summ-fact} shows the average factual
accuracy rating for each model. We find that rejection sampling outperforms
\emph{all} baselines, including the current gold standard of beam search, and
matches the human reference level of factual accuracy.

Although it may seem surprising that loss truncation and rejection sampling
together can achieve the same factual accuracy score as humans, recall that over
34\% of the dataset consists of titles which have facts that are not contained
in the article. The loss truncation approach biases the model towards learning
only the easily predicted (and likely factually accurate) titles.

\subsection{Loss Truncation Produces Diverse Outputs}

\begin{table*}[t!]
  \centering
  \small
  \begin{tabularx}{\linewidth}{l|X}
    Method & Example \\
    \hline
    \hline
    Context & at least \#\# people have been killed and more than \#\#,\#\#\# made homeless by floods that swept across southern africa in the past week , striking a region already grappling with severe food shortages . \\
    \hline
    Gold & floods kill \#\# in famine-hit southern africa \\
    \hline
    \hline
    Loss truncation & at least \#\# people killed \#\#,\#\#\# evacuated in floods in southern african region \\
                    & floods that sweep parts of africa kill at least \#\# \\
    \hline
    Beam & flooding hits southern africa as deaths rise \\
    \hline
    Full sampling & \colorred{child farming} stalls in southern africa \\
                  & \colorred{earthquake} kills \#\# in southern africa \\
    \hline
    top-$p$ ($p=0.9$) & \colorred{torrential rains} prompt warnings in southern africa \\
                      & toll nears \#\# in southern africa \\
    \hline
    top-$k$ ($k=2$)   & at least \#\# killed \#\#,\#\#\# homeless in southern africa floods \\
                      & at least \#\# dead \#\#,\#\#\# homeless as floods hit southern africa \\
  \end{tabularx}
  \vspace{-0.5em}
  \caption{Examples of generations for various baselines and loss truncation
  (two replicates shown for sampled outputs). As shown, loss truncation can
  achieve diverse and high quality outputs. In contrast, baselines either are
  not diverse (beam, top-$k$) or poor quality (full sampling, top-$p$). We color
  incorrect facts in red.}
  \label{table:trunc-examples}
\end{table*}

Finally, one of the benefits of optimizing for distinguishability is that it naturally
optimizes for both diversity and quality. Manually examining outputs from the models, we
find that directly sampling from the loss truncated model often produces high quality and diverse outputs. We show examples of
generated outputs for baselines and loss truncation in
Table~\ref{table:trunc-examples}. Loss truncation uses different
phrasings (`at least \# killed', and `floods sweep') while top-$k$ follows a
nearly templated pattern with a few changes to the words which appear. Top-$p$ and
direct sampling both have diverse phrasings, but also hallucinate facts
(`earthquake' in sampling and `torrential rains' in top-$p$ sampling).

\subsection{Loss Truncation can Outperform on Automated Metrics}

%
%

While our primary evaluation metrics are human evaluations (HUSE and factuality),
we additionally investigate automated metrics to further contextualize our
results. For summarization, we used ROUGE-L and for E2E we use BLEU score for
the automated metrics.

For summarization, the ROUGE-L scores for loss truncation and entropy-matched
top-$k$ and top-$p$ decoding were 23.2, 22.8, and 22.8 respectively. While loss
truncation does not substantially improve ROUGE-L, we see that it still
outperforms baselines. We do not expect reference-based evaluations to fully
capture the benefits of loss truncation, as these metrics encourage the models
to fully imitate the data distribution -- including invalid and hallucinated
examples.

For E2E, the BLEU scores for loss truncation and the baseline were 0.72 and 0.64
respectively. We confirmed that the baseline model for the E2E task achieves a
similar score as reported by \citet{balakrishnan2019constrained}. Perhaps
surprisingly, improving BLEU score to 0.72 almost closes the gap to using
complex tree-structured semantic representations, which achieves a BLEU score of
0.74 \cite{balakrishnan2019constrained}.

We further show that loss truncation is not sensitive to the hyperparameter
$c$ on automated metrics in Appendix~\ref{sec:sensitivity-to-c} and provide a
preliminary investigation of combining loss truncation and alternative decoders
in Appendix~\ref{sec:trunc-and-decoder}.

\section{Related Work}
\label{sec:related-work}

\minihead{Decoder-based diversity}
Researchers have proposed a variety of models for text generation
\cite{radford2019language, keskar2019ctrl, sutskever2014sequence}. These models
generate text using decoding methods such as beam search.
While beam search is generally thought of as the gold standard \cite{tillmann2003word}, it can produce generic and repetitive 
outputs \cite{holtzman2019curious}. To achieve diversity, top-$k$
\cite{fan2018hierarchical} and top-$p$ \cite{holtzman2019curious} sampling
stochastically decodes the outputs after restricting the output space
to avoid low-quality outputs.

While these techniques can improve generation quality, they rely on models
trained via log loss, which we show can result in undesired behavior that cannot
be fixed post-hoc.  Our work is complementary to existing work on decoders by
proposing a loss that can improve the probabilistic models which these decoders
operate on.

\minihead{Loss modifications}
Prior work has identified specific issues in generative models, such as
repetitiveness, and proposed loss modifications to address these specific
issues in the context of long text generation \cite{welleck2019neural, holtzman2018learning}. In contrast, we identify
an issue with the widely used log loss, and propose loss truncation, which does not require a task- and
issue-specific modification. Many of the penalties and decoding techniques
proposed in these earlier works can be combined with truncated log loss to obtain
models that are more robust to noisy references.

Contemporaneous with our work, \citet{tian2019sticking} propose an attention
weight approach to improving generation faithfulness via decoder and loss
modifications. Our work complements this by providing a conceptual basis for
improving faithfulness by ignoring examples (i.e., optimizing distinguishability), and providing a simple and general loss. We consider complex, model dependent loss truncation methods for optimizing distinguishability to be exciting future work.

Other generation methods optimize for task-specific losses \cite{och2003minimum,
shen2015minimum}. Task specific losses are not known in many cases and thus we
require an effective task-agnostic loss, e.g., log loss or TV. We
show that TV acts as a useful task-agnostic goodness of fit measure, 
and we provide an improved alternative to log loss.

\minihead{GANs}
GANs have been proposed to learn models that minimize distinguishability
\cite{li2017adversarial, rajeswar2017adversarial, dai2017towards}. While GANs
have been successful in generating images \cite{goodfellow2014generative,
brock2018large}, GANs remaining challenging to optimize for text due to the
discrete nature of text.  Our findings match earlier reports that GANs
underperform log loss trained sequence-to-sequence models
\cite{caccia2018language}. In this work, we show that better training
methods for distinguishability can arise from modifying the standard log loss
via truncation.

\minihead{Robust learning}
Robust learning is the study of learning in the face of outliers
\cite{tukey1960survey, donoho1982breakdown, huber1992robust}. Our work is
related to the $\epsilon$-contamination model, in which an $\epsilon$ fraction
of the data has been modified, potentially by an adversary
\cite{diakonikolas2018sever}. Our work shows that robust learning under log loss
can result in improved empirical performance and bounds on distinguishability.

While there are a number of effective approaches to robust learning \cite{diakonikolas2018sever, fischler1981random}, we focus on a simple truncation procedure as it is one of the only procedures scaleable enough to apply on large-scale generation datasets. Our work shows that more effective, scalable robust learning procedures can help improve natural language generation methods.

\section{Conclusion}
\label{sec:conclusion}

In this work, we show that log loss is not robust to noise, which can in turn
cause undesired behavior, such as hallucinating facts in summarization. In
response, we propose loss truncation, a robust training method that optimizes
for distinguishability of generated samples. We additionally propose a
sequence-level rejection sampling scheme to generate high quality sequences. We
show that loss truncation outperforms a range of baselines (including beam
search, top-$p$, top-$k$, and full sampling) on distinguishability. We
additionally show that rejection sampling outperforms all baselines, including
beam search, on generating factual summaries. These results suggest that robust
learning in the form of truncating the log loss can complement model-based approaches to
faithful generation by ignoring invalid and undesired references.

\bibliography{nlg-acl}

\begin{thebibliography}{50}
\expandafter\ifx\csname natexlab\endcsname\relax\def\natexlab#1{#1}\fi

\bibitem[{Balakrishnan et~al.(2019)Balakrishnan, Rao, Upasani, White, and
  Subba}]{balakrishnan2019constrained}
Anusha Balakrishnan, Jinfeng Rao, Kartikeya Upasani, Michael White, and Rajen
  Subba. 2019.
\newblock Constrained decoding for neural nlg from compositional
  representations in task-oriented dialogue.
\newblock \emph{arXiv preprint arXiv:1906.07220}.

\bibitem[{Beigman and Klebanov(2009)}]{beigman2009learning}
Eyal Beigman and Beata~Beigman Klebanov. 2009.
\newblock Learning with annotation noise.
\newblock In \emph{Proceedings of the Joint Conference of the 47th Annual
  Meeting of the ACL and the 4th International Joint Conference on Natural
  Language Processing of the AFNLP: Volume 1-Volume 1}, pages 280--287.
  Association for Computational Linguistics.

\bibitem[{Brock et~al.(2018)Brock, Donahue, and Simonyan}]{brock2018large}
Andrew Brock, Jeff Donahue, and Karen Simonyan. 2018.
\newblock Large scale gan training for high fidelity natural image synthesis.
\newblock \emph{arXiv preprint arXiv:1809.11096}.

\bibitem[{Caccia et~al.(2018)Caccia, Caccia, Fedus, Larochelle, Pineau, and
  Charlin}]{caccia2018language}
Massimo Caccia, Lucas Caccia, William Fedus, Hugo Larochelle, Joelle Pineau,
  and Laurent Charlin. 2018.
\newblock Language gans falling short.
\newblock \emph{arXiv preprint arXiv:1811.02549}.

\bibitem[{Csiszar and K{\"o}rner(2011)}]{csiszar2011information}
Imre Csiszar and J{\'a}nos K{\"o}rner. 2011.
\newblock \emph{Information theory: coding theorems for discrete memoryless
  systems}.
\newblock Cambridge University Press.

\bibitem[{Dai et~al.(2017)Dai, Fidler, Urtasun, and Lin}]{dai2017towards}
Bo~Dai, Sanja Fidler, Raquel Urtasun, and Dahua Lin. 2017.
\newblock Towards diverse and natural image descriptions via a conditional gan.
\newblock In \emph{Proceedings of the IEEE International Conference on Computer
  Vision}, pages 2970--2979.

\bibitem[{Dhingra et~al.(2019)Dhingra, Faruqui, Parikh, Chang, Das, and
  Cohen}]{dhingra2019handling}
Bhuwan Dhingra, Manaal Faruqui, Ankur Parikh, Ming-Wei Chang, Dipanjan Das, and
  William~W Cohen. 2019.
\newblock Handling divergent reference texts when evaluating table-to-text
  generation.
\newblock \emph{arXiv preprint arXiv:1906.01081}.

\bibitem[{Diakonikolas et~al.(2018)Diakonikolas, Kamath, Kane, Li, Steinhardt,
  and Stewart}]{diakonikolas2018sever}
Ilias Diakonikolas, Gautam Kamath, Daniel~M Kane, Jerry Li, Jacob Steinhardt,
  and Alistair Stewart. 2018.
\newblock Sever: A robust meta-algorithm for stochastic optimization.
\newblock \emph{arXiv preprint arXiv:1803.02815}.

\bibitem[{Donoho et~al.(1988)Donoho, Liu et~al.}]{donoho1988automatic}
David~L Donoho, Richard~C Liu, et~al. 1988.
\newblock The" automatic" robustness of minimum distance functionals.
\newblock \emph{The Annals of Statistics}, 16(2):552--586.

\bibitem[{Donoho(1982)}]{donoho1982breakdown}
DL~Donoho. 1982.
\newblock Breakdown properties of multivariate location estimators.
\newblock \emph{The Annals of Statistics}.

\bibitem[{Du{\v{s}}ek et~al.(2019)Du{\v{s}}ek, Novikova, and
  Rieser}]{dusek2019e2e}
Ond\v{r}ej Du{\v{s}}ek, Jekaterina Novikova, and Verena Rieser. 2019.
\newblock Evaluating the state-of-the-art of end-to-end natural language
  generation: {The} {E2E} {NLG} {Challenge}.
\newblock \emph{arXiv preprint arXiv:1901.11528}.

\bibitem[{Fan et~al.(2018)Fan, Lewis, and Dauphin}]{fan2018hierarchical}
Angela Fan, Mike Lewis, and Yann Dauphin. 2018.
\newblock Hierarchical neural story generation.
\newblock \emph{ACL}.

\bibitem[{Fischler and Bolles(1981)}]{fischler1981random}
Martin~A Fischler and Robert~C Bolles. 1981.
\newblock Random sample consensus: a paradigm for model fitting with
  applications to image analysis and automated cartography.
\newblock \emph{Communications of the ACM}, 24(6):381--395.

\bibitem[{Gamon et~al.(2005)Gamon, Aue, and Smets}]{gamon2005sentence}
Michael Gamon, Anthony Aue, and Martine Smets. 2005.
\newblock Sentence-level mt evaluation without reference translations: Beyond
  language modeling.
\newblock In \emph{Proceedings of EAMT}, pages 103--111.

\bibitem[{Gehrmann et~al.(2018)Gehrmann, Deng, and Rush}]{gehrmann2018bottom}
Sebastian Gehrmann, Yuntian Deng, and Alexander Rush. 2018.
\newblock Bottom-up abstractive summarization.
\newblock In \emph{Proceedings of the 2018 Conference on Empirical Methods in
  Natural Language Processing}, pages 4098--4109.

\bibitem[{Gehrmann et~al.(2019)Gehrmann, Strobelt, and Rush}]{gehrmann2019gltr}
Sebastian Gehrmann, Hendrik Strobelt, and Alexander~M Rush. 2019.
\newblock Gltr: Statistical detection and visualization of generated text.
\newblock \emph{arXiv preprint arXiv:1906.04043}.

\bibitem[{Goodfellow et~al.(2014)Goodfellow, Pouget-Abadie, Mirza, Xu,
  Warde-Farley, Ozair, Courville, and Bengio}]{goodfellow2014generative}
Ian Goodfellow, Jean Pouget-Abadie, Mehdi Mirza, Bing Xu, David Warde-Farley,
  Sherjil Ozair, Aaron Courville, and Yoshua Bengio. 2014.
\newblock Generative adversarial nets.
\newblock In \emph{Advances in neural information processing systems}, pages
  2672--2680.

\bibitem[{Hashimoto et~al.(2019)Hashimoto, Zhang, and
  Liang}]{hashimoto2019unifying}
Tatsunori~B Hashimoto, Hugh Zhang, and Percy Liang. 2019.
\newblock Unifying human and statistical evaluation for natural language
  generation.
\newblock \emph{North American Chapter of the Association for Computational
  Linguistics}.

\bibitem[{Holtzman et~al.(2018)Holtzman, Buys, Forbes, Bosselut, Golub, and
  Choi}]{holtzman2018learning}
Ari Holtzman, Jan Buys, Maxwell Forbes, Antoine Bosselut, David Golub, and
  Yejin Choi. 2018.
\newblock Learning to write with cooperative discriminators.
\newblock \emph{arXiv preprint arXiv:1805.06087}.

\bibitem[{Holtzman et~al.(2019)Holtzman, Buys, Forbes, and
  Choi}]{holtzman2019curious}
Ari Holtzman, Jan Buys, Maxwell Forbes, and Yejin Choi. 2019.
\newblock The curious case of neural text degeneration.
\newblock \emph{arXiv preprint arXiv:1904.09751}.

\bibitem[{Huber(1992)}]{huber1992robust}
Peter~J Huber. 1992.
\newblock Robust estimation of a location parameter.
\newblock In \emph{Breakthroughs in statistics}, pages 492--518. Springer.

\bibitem[{Keskar et~al.(2019)Keskar, McCann, Varshney, Xiong, and
  Socher}]{keskar2019ctrl}
Nitish~Shirish Keskar, Bryan McCann, Lav~R Varshney, Caiming Xiong, and Richard
  Socher. 2019.
\newblock Ctrl: A conditional transformer language model for controllable
  generation.
\newblock \emph{arXiv preprint arXiv:1909.05858}.

\bibitem[{Klebanov and Beigman(2010)}]{klebanov2010some}
Beata~Beigman Klebanov and Eyal Beigman. 2010.
\newblock Some empirical evidence for annotation noise in a benchmarked
  dataset.
\newblock In \emph{Human Language Technologies: The 2010 Annual Conference of
  the North American Chapter of the Association for Computational Linguistics},
  pages 438--446. Association for Computational Linguistics.

\bibitem[{Klein et~al.(2017)Klein, Kim, Deng, Senellart, and
  Rush}]{klein2017opennmt}
Guillaume Klein, Yoon Kim, Yuntian Deng, Jean Senellart, and Alexander~M Rush.
  2017.
\newblock Opennmt: Open-source toolkit for neural machine translation.
\newblock \emph{arXiv preprint arXiv:1701.02810}.

\bibitem[{Kry{\'s}ci{\'n}ski et~al.(2019)Kry{\'s}ci{\'n}ski, Keskar, McCann,
  Xiong, and Socher}]{kryscinski2019neural}
Wojciech Kry{\'s}ci{\'n}ski, Nitish~Shirish Keskar, Bryan McCann, Caiming
  Xiong, and Richard Socher. 2019.
\newblock Neural text summarization: A critical evaluation.
\newblock \emph{arXiv preprint arXiv:1908.08960}.

\bibitem[{Lee et~al.(2018)Lee, Firat, Agarwal, Fannjiang, and
  Sussillo}]{lee2018hallucinations}
Katherine Lee, Orhan Firat, Ashish Agarwal, Clara Fannjiang, and David
  Sussillo. 2018.
\newblock Hallucinations in neural machine translation.
\newblock \emph{Interpretability and Robustness in Audio, Speech, and Language
  Workshop}.

\bibitem[{Li et~al.(2017)Li, Monroe, Shi, Jean, Ritter, and
  Jurafsky}]{li2017adversarial}
Jiwei Li, Will Monroe, Tianlin Shi, S{\'e}bastien Jean, Alan Ritter, and Dan
  Jurafsky. 2017.
\newblock Adversarial learning for neural dialogue generation.
\newblock \emph{arXiv preprint arXiv:1701.06547}.

\bibitem[{Lin and Hovy(2003)}]{lin2003automatic}
Chin-Yew Lin and Eduard Hovy. 2003.
\newblock Automatic evaluation of summaries using n-gram co-occurrence
  statistics.
\newblock In \emph{Proceedings of the 2003 Human Language Technology Conference
  of the North American Chapter of the Association for Computational
  Linguistics}, pages 150--157.

\bibitem[{Lin et~al.(2014)Lin, Maire, Belongie, Hays, Perona, Ramanan,
  Doll{\'a}r, and Zitnick}]{lin2014microsoft}
Tsung-Yi Lin, Michael Maire, Serge Belongie, James Hays, Pietro Perona, Deva
  Ramanan, Piotr Doll{\'a}r, and C~Lawrence Zitnick. 2014.
\newblock Microsoft coco: Common objects in context.
\newblock In \emph{European conference on computer vision}, pages 740--755.
  Springer.

\bibitem[{Nallapati et~al.(2016)Nallapati, Zhou, Gulcehre, Xiang
  et~al.}]{nallapati2016abstractive}
Ramesh Nallapati, Bowen Zhou, Caglar Gulcehre, Bing Xiang, et~al. 2016.
\newblock Abstractive text summarization using sequence-to-sequence rnns and
  beyond.
\newblock \emph{arXiv preprint arXiv:1602.06023}.

\bibitem[{Novikova et~al.(2017)Novikova, Du{\v{s}}ek, Curry, and
  Rieser}]{novikova2017we}
Jekaterina Novikova, Ond{\v{r}}ej Du{\v{s}}ek, Amanda~Cercas Curry, and Verena
  Rieser. 2017.
\newblock Why we need new evaluation metrics for nlg.
\newblock \emph{arXiv preprint arXiv:1707.06875}.

\bibitem[{Och(2003)}]{och2003minimum}
Franz~Josef Och. 2003.
\newblock Minimum error rate training in statistical machine translation.
\newblock In \emph{Proceedings of the 41st Annual Meeting on Association for
  Computational Linguistics-Volume 1}, pages 160--167. Association for
  Computational Linguistics.

\bibitem[{Papineni et~al.(2002)Papineni, Roukos, Ward, and
  Zhu}]{papineni2002bleu}
Kishore Papineni, Salim Roukos, Todd Ward, and Wei-Jing Zhu. 2002.
\newblock Bleu: a method for automatic evaluation of machine translation.
\newblock In \emph{Proceedings of the 40th annual meeting on association for
  computational linguistics}, pages 311--318. Association for Computational
  Linguistics.

\bibitem[{Puzikov and Gurevych(2018)}]{puzikov2018e2e}
Yevgeniy Puzikov and Iryna Gurevych. 2018.
\newblock E2e nlg challenge: Neural models vs. templates.
\newblock In \emph{Proceedings of the 11th International Conference on Natural
  Language Generation}, pages 463--471.

\bibitem[{Radford et~al.(2019)Radford, Wu, Child, Luan, Amodei, and
  Sutskever}]{radford2019language}
Alec Radford, Jeffrey Wu, Rewon Child, David Luan, Dario Amodei, and Ilya
  Sutskever. 2019.
\newblock Language models are unsupervised multitask learners.
\newblock \emph{OpenAI Blog}, 1(8).

\bibitem[{Rajeswar et~al.(2017)Rajeswar, Subramanian, Dutil, Pal, and
  Courville}]{rajeswar2017adversarial}
Sai Rajeswar, Sandeep Subramanian, Francis Dutil, Christopher Pal, and Aaron
  Courville. 2017.
\newblock Adversarial generation of natural language.
\newblock \emph{arXiv preprint arXiv:1705.10929}.

\bibitem[{Roemmele(2016)}]{roemmele2016writing}
Melissa Roemmele. 2016.
\newblock Writing stories with help from recurrent neural networks.
\newblock In \emph{Thirtieth AAAI Conference on Artificial Intelligence}.

\bibitem[{Rush et~al.(2017)Rush, Harvard, Chopra, and Weston}]{rush2017neural}
Alexander~M Rush, SEAS Harvard, Sumit Chopra, and Jason Weston. 2017.
\newblock A neural attention model for sentence summarization.
\newblock In \emph{ACLWeb. Proceedings of the 2015 Conference on Empirical
  Methods in Natural Language Processing}.

\bibitem[{See et~al.(2017)See, Liu, and Manning}]{see2017get}
Abigail See, Peter~J Liu, and Christopher~D Manning. 2017.
\newblock Get to the point: Summarization with pointer-generator networks.
\newblock \emph{arXiv preprint arXiv:1704.04368}.

\bibitem[{Shen et~al.(2015)Shen, Cheng, He, He, Wu, Sun, and
  Liu}]{shen2015minimum}
Shiqi Shen, Yong Cheng, Zhongjun He, Wei He, Hua Wu, Maosong Sun, and Yang Liu.
  2015.
\newblock Minimum risk training for neural machine translation.
\newblock \emph{arXiv preprint arXiv:1512.02433}.

\bibitem[{Sutskever et~al.(2014)Sutskever, Vinyals, and
  Le}]{sutskever2014sequence}
Ilya Sutskever, Oriol Vinyals, and Quoc~V Le. 2014.
\newblock Sequence to sequence learning with neural networks.
\newblock \emph{Advances in NIPS}.

\bibitem[{Theis et~al.(2016)Theis, Oord, and Bethge}]{theis2016note}
Lucas Theis, A{\"a}ron van~den Oord, and Matthias Bethge. 2016.
\newblock A note on the evaluation of generative models.
\newblock \emph{ICLR}.

\bibitem[{Tian et~al.(2019)Tian, Narayan, Sellam, and
  Parikh}]{tian2019sticking}
Ran Tian, Shashi Narayan, Thibault Sellam, and Ankur~P Parikh. 2019.
\newblock Sticking to the facts: Confident decoding for faithful data-to-text
  generation.
\newblock \emph{arXiv preprint arXiv:1910.08684}.

\bibitem[{Tillmann and Ney(2003)}]{tillmann2003word}
Christoph Tillmann and Hermann Ney. 2003.
\newblock Word reordering and a dynamic programming beam search algorithm for
  statistical machine translation.
\newblock \emph{Computational linguistics}, 29(1):97--133.

\bibitem[{Tukey(1960)}]{tukey1960survey}
John~W Tukey. 1960.
\newblock A survey of sampling from contaminated distributions.
\newblock \emph{Contributions to probability and statistics}, pages 448--485.

\bibitem[{Wang and Lee(2018)}]{wang2018learning}
Yau-Shian Wang and Hung-Yi Lee. 2018.
\newblock Learning to encode text as human-readable summaries using generative
  adversarial networks.
\newblock \emph{arXiv preprint arXiv:1810.02851}.

\bibitem[{Welleck et~al.(2019)Welleck, Kulikov, Roller, Dinan, Cho, and
  Weston}]{welleck2019neural}
Sean Welleck, Ilia Kulikov, Stephen Roller, Emily Dinan, Kyunghyun Cho, and
  Jason Weston. 2019.
\newblock Neural text generation with unlikelihood training.
\newblock \emph{arXiv preprint arXiv:1908.04319}.

\bibitem[{Wiseman et~al.(2017)Wiseman, Shieber, and
  Rush}]{wiseman2017challenges}
Sam Wiseman, Stuart~M Shieber, and Alexander~M Rush. 2017.
\newblock Challenges in data-to-document generation.
\newblock \emph{arXiv preprint arXiv:1707.08052}.

\bibitem[{Zellers et~al.(2019)Zellers, Holtzman, Rashkin, Bisk, Farhadi,
  Roesner, and Choi}]{zellers2019defending}
Rowan Zellers, Ari Holtzman, Hannah Rashkin, Yonatan Bisk, Ali Farhadi,
  Franziska Roesner, and Yejin Choi. 2019.
\newblock Defending against neural fake news.
\newblock \emph{arXiv preprint arXiv:1905.12616}.

\bibitem[{Zhou et~al.(2019)Zhou, Gordon, Krishna, Narcomey, Fei-Fei, and
  Bernstein}]{zhou2019hype}
Sharon Zhou, Mitchell Gordon, Ranjay Krishna, Austin Narcomey, Li~F Fei-Fei,
  and Michael Bernstein. 2019.
\newblock Hype: A benchmark for human eye perceptual evaluation of generative
  models.
\newblock In \emph{Advances in Neural Information Processing Systems}, pages
  3444--3456.

\end{thebibliography}
\bibliographystyle{acl_natbib}

\clearpage

\appendix

\section{Examples of Titles and Generations}
\label{sec:generation-ex}

\minihead{Examples of ground truth titles}
We present examples of titles in Figure~\ref{fig:hallucination-examples} that
require factual hallucination and can be directly entailed from context.

\begin{figure}[t!]
  \begin{subfigure}[t]{0.99\columnwidth}
    \begin{mdframed}
    \textbf{Context:} Donna Shalala is sporting a mustache to promote public
    health. \\
    \textbf{Title:} Milk on Her Lip Shalala Raises Eyebrows
    \end{mdframed}
    \caption{Example of a title that requires hallucinating new facts, e.g.,
    ``Milk on Her Lip'' and ``raises eyebrows''.}
  \end{subfigure}
  \begin{subfigure}[t]{0.99\columnwidth}
    \begin{mdframed}
    \textbf{Context:} Southwest China's Sichuan province has decided to build an
    inter-city high-tech industrial belt to serve development of Western China.
    \\
    \textbf{Title:} Sichuan to Build High-Tech Industrial Belt
    \end{mdframed}
    \caption{Example of a title that can be directly generated from the context.}
  \end{subfigure}
  \caption{Examples of titles that require hallucinating new facts and titles
  that are directly entailed from context.}
  \label{fig:hallucination-examples}
\end{figure}

\minihead{Examples of generated titles}
We present examples of titles that from rejection sampling that are selected and
that were rejected in sampling in Figure~\ref{fig:rejection-examples}. As shown,
rejected titles tend to be of lower quality.

\begin{figure}[t!]
  \begin{subfigure}[t]{0.99\columnwidth}
    \begin{mdframed}
    \textbf{Context:} At least two people have tested positive for the bird flu
    virus in Eastern Turkey, health minister Recep Akdag told a news conference
    Wednesday.
    \\
    \textbf{Ground truth:} Two test positive for bird flu virus in Turkey
    \\
    \textbf{Selected sample:} Two reported positive for bird flu in Eastern
    Turkey
    \\
    \textbf{Rejected sample:} Two officials fail to get good for bird flu in
    Eastern Turkey
    \end{mdframed}
    \caption{Example 1.}
  \end{subfigure}
  \begin{subfigure}[t]{0.99\columnwidth}
    \begin{mdframed}
    \textbf{Context:} British investment fund Fidelity has increased its stake
    in Puma, the German maker of sportswear and equipment, to just over five
    percent, Puma said on Thursday.
    \\
    \textbf{Ground truth:} Private equity firm Fidelity raises stake in Puma to over five pct
    \\
    \textbf{Selected sample:} Fidelity increases stake in Puma
    \\
    \textbf{Rejected sample:} Boost higher first-half stake in Puma says Puma
    \end{mdframed}
    \caption{Example 2.}
  \end{subfigure}
  \caption{Examples of sampled titles that were selected and rejected in
  rejection sampling at $\alpha=0.1$.}
  \label{fig:rejection-examples}
\end{figure}

\section{Proof of Lemma and Proposition}
\label{sec:proofs}

\minihead{Lemma}
We prove the lemma that all elements in $\truncset{p}$ are close to $p$ in total
variation.

\begin{lem}\label{lem:tvbound}
\[ \sup_{q_0\in \truncset{p}} |q_0 - p|_{\text{TV}} \leq c\]
\end{lem}
\begin{proof}
  By definition of $\truncset{p}$, for any $q_0$ there exists a $q_1$ such that
  $p = c q_1 + (1 - c) q_0$ so,
\[
  |q_0 - p|_{\text{TV}}= \left|c q_0 - c q_1\right|_{\text{TV}}\leq c
\]
  \end{proof}

\minihead{Proposition}
We prove that the truncated log loss bounds total variation.

\begin{proof}
  \begin{align}
&|\pmodel - \pref|_{\text{TV}}^2\\
  &\leq (|\pmodel - \ptrunc|_{\text{TV}} + |\ptrunc - \pref|_{\text{TV}})^2 \\
  &\leq \frac{1}{2} \KL(\ptrunc || \pmodel) + 2c + c^2 \label{eq:bound}
  \end{align}
  which follows from the triangle inequality, Pinsker's inequality, and using Lemma~\ref{lem:tvbound} to bound the remaining terms by $c$.
\end{proof}

\section{Hyperparameters}
\label{sec:hyperparameters}

\minihead{Summarization model hyperparameters}
We used a standard OpenNMT-py model with global attention for all
sequence-to-sequence experiments \cite{klein2017opennmt}. It has a single LSTM
layer in the encoder and two in the decoder.

For the baseline model, we train for 200,000 steps with SGD and an initial
learning rate of $1$. For the loss truncated model, we hotstart with 100,000
minibatch updates and subsequently with 100,000 minibatch updates with the
truncated loss with an initial learning rate of $0.1$.

\minihead{$k$ and $p$ selection}
A key parameter in top-$k$ and top-$p$ sampling are $k$ and $p$ respectively.
These parameters trade off between diversity and quality. To select these
values, we chose values of $k$ and $p$ that had similar entropies to our model
trained with loss truncation.

Specifically, $k=100$ and $p=0.9$ matched loss truncation at $c=0.6$ for
summarization (entropies of $18.08$, $20.01$, and $17.93$ respectively).
$k=2$ and $p=0.4$ matched rejection sampling for summarization at $c=0.6, \alpha=0.1$
(entropies of $3.71$, $4.02$, and $3.84$ respectively).

\section{Crowd Worker Setup and Prompts}
\label{sec:worker-prompts}

\minihead{Crowdsourcing setup}
For all human evaluations, we used Amazon Mechanical Turk (all prompts
shown below). We sampled 312 context/title pairs to measure HUSE. For each
generated title, we asked 9 crowd workers to measure the typicality of the
generated title, as in \citet{hashimoto2019unifying}.  Each crowd worker responded to
24 generated titles.

For measuring factuality, we sampled 312 examples and for each example, we
asked two crowd workers how much information in the generated title was present in the article.

\begin{figure}
  \begin{subfigure}{\columnwidth}
    \includegraphics[width=\columnwidth]{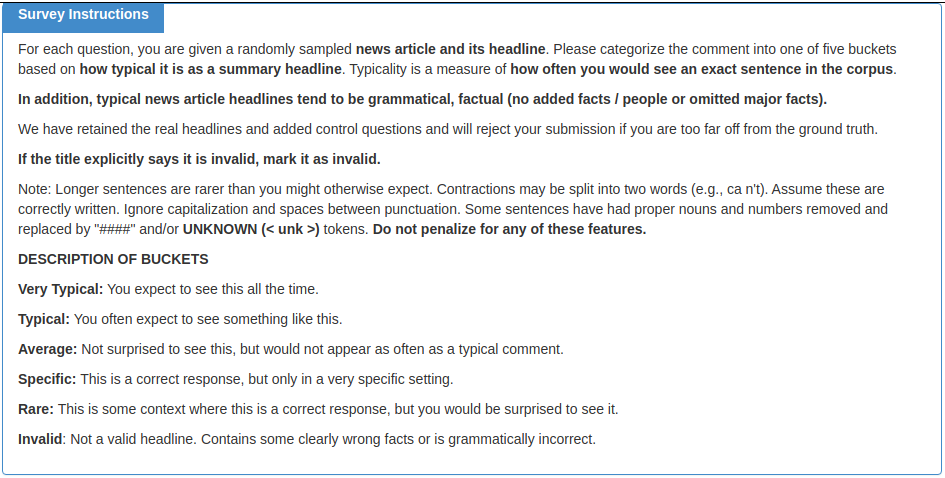}
    \caption{Prompt for measuring HUSE.}
  \end{subfigure}
  \begin{subfigure}{\columnwidth}
    \includegraphics[width=\columnwidth]{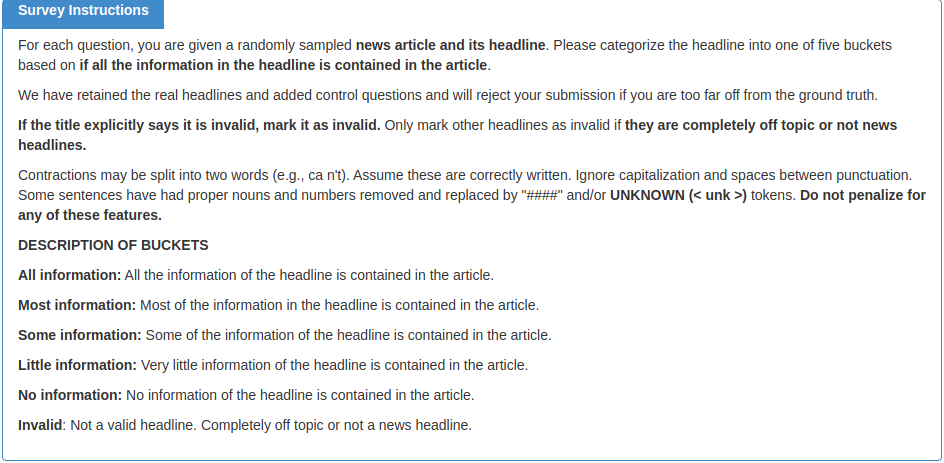}
    \caption{Prompt for measuring factuality.}
  \end{subfigure}
  \caption{Prompts for measuring HUSE and factuality.}
  \label{fig:mturk-prompts}
\end{figure}

\minihead{Prompts}
We show crowd worker prompts for measuring HUSE and factuality in
Figure~\ref{fig:mturk-prompts}. The HUSE prompt was directly taken from
\citet{hashimoto2019unifying} with an extra control.

\section{Further experiments}
\subsection{Sensitivity to $c$}
\label{sec:sensitivity-to-c}

\begin{table}[t!]
  \centering
  \begin{tabular}{ll}
    Condition & ROUGE-L \\ \hline \hline
    Truncation, $c=0.9$ & 24.3 \\
    Truncation, $c=0.8$ & \textbf{24.9} \\
    Truncation, $c=0.7$ & 24.0 \\
    Truncation, $c=0.6$ & 23.2 \\
    \hline
    top-$k=100$         & 22.8 \\
    top-$p=0.9$         & 22.8
  \end{tabular}
  \caption{ROUGE-L scores for loss truncation at various $c$ and
  entropy-matched top-$k$ and top-$p$ decoding for summarization. As shown, loss
  truncation outperforms on ROUGE-L for a range of $c$.}
  \label{table:summ-vary-c}
\end{table}

\begin{table}[t!]
  \centering
  \begin{tabular}{ll}
    Condition & BLEU \\ \hline \hline
    Truncation, $c=0.9$ & \textbf{0.72} \\
    Truncation, $c=0.8$ & 0.71 \\
    Truncation, $c=0.7$ & 0.70 \\
    Truncation, $c=0.6$ & 0.69 \\
    Truncation, $c=0.5$ & 0.69 \\
    \hline
    Baseline            & 0.64 \\
    \textbf{0.72} & 0.64
  \end{tabular}
  \caption{BLEU scores for loss truncation at various $c$ and the
  baseline model on the E2E task. As shown, loss truncation outperforms
  the baseline on BLEU score at a range of hyperparameters.}
  \label{table:e2e-vary-c}
\end{table}

We investigate the sensitivity of loss truncation to the hyperparameter $c$. To
do so, we vary $c$ and measure ROUGE-L and BLEU scores, for summarization and
E2E respectively.

We show results for summarization in Table~\ref{table:summ-vary-c} and E2E in
Table~\ref{table:e2e-vary-c} along with baselines. As shown, truncation
outperforms on automated metrics on a variety of hyperparameter settings on
automated metrics.  We leave a full investigation of sensitivity to $c$ as
future work.

\subsection{Combining Loss Truncation and Decoders}
\label{sec:trunc-and-decoder}

\begin{table}
  \centering
  \begin{tabular}{ll}
  Condition & ROUGE-L \\ \hline \hline
  Log-loss, beam & \textbf{41.4} \\
  Log-loss, full sampling & 27.9 \\
  \hline
  Truncation, top-$k=100$ & 33.4 \\
  Truncation, top-$k=2$   & 38.9 \\
  Truncation, top-$p=0.9$ & 35.1 \\
  Truncation, top-$p=0.1$ & 40.9
  \end{tabular}
  \caption{Loss truncation combined with top-$k$ and top-$p$ decoding.}
  \label{table:trunc-decoders}
\end{table}

As loss truncation is a training method, it can be combined with alternative
methods of decoding at inference time. As such, we perform a preliminary
investigation of using top-$k$ and top-$p$ decoding with loss truncation.

We show ROUGE-L of loss truncation combined with various decoders and baselines
for summarization in Table~\ref{table:trunc-decoders}. As shown, top-$k$ and
top-$p$ decoding work with loss truncation and can improve sample quality.

\end{document}